\documentclass[sigconf]{acmart}

\usepackage{booktabs} 
\usepackage[utf8]{inputenc}

\usepackage{url}  
\usepackage{graphicx}  
\usepackage{amsmath}
\usepackage{amsthm}
\usepackage{amssymb}
\usepackage{color}
\usepackage{subfigure}
\usepackage{multirow}
\usepackage{enumitem}
\usepackage{algorithm}
\usepackage{algorithmic}
\usepackage{caption}
\usepackage{soul}
\usepackage{todonotes}

\newtheorem{prop}{Proposition}
\newcommand{\code}[1]{\texttt{#1}}

\newcommand{\vect}[1]{\mathbf{#1}}

\usepackage{array}
\newcolumntype{L}[1]{>{\raggedright\let\newline\\\arraybackslash\hspace{0pt}}m{#1}}
\newcolumntype{C}[1]{>{\centering\let\newline  \\\arraybackslash\hspace{0pt}}m{#1}}
\newcolumntype{R}[1]{>{\raggedleft\let\newline \\\arraybackslash\hspace{0pt}}m{#1}}

\begin{document}

\copyrightyear{2019} 
\acmYear{2019} 
\setcopyright{acmcopyright}
\acmConference[KDD '19]{The 25th ACM SIGKDD Conference on Knowledge Discovery and Data Mining}{August 4--8, 2019}{Anchorage, AK, USA}
\acmBooktitle{The 25th ACM SIGKDD Conference on Knowledge Discovery and Data Mining (KDD '19), August 4--8, 2019, Anchorage, AK, USA}
\acmPrice{15.00}
\acmDOI{10.1145/3292500.3330679}
\acmISBN{978-1-4503-6201-6/19/08}

\settopmatter{printacmref=true}
\fancyhead{}

\title[AutoCross]{AutoCross: Automatic Feature Crossing for Tabular Data in Real-World Applications}



\author{Yuanfei Luo}
\email{luoyuanfei@4paradigm.com}
\affiliation{
	\institution{4Paradigm Inc., Beijing, China}
}
\author{Mengshuo Wang}
\email{wangmengshuo@4paradigm.com}
\affiliation{
	\institution{4Paradigm Inc., Beijing, China}
}
\author{Hao Zhou}
\email{zhouhao@4paradigm.com}
\affiliation{
	\institution{4Paradigm Inc., Beijing, China}
}
\author{Quanming Yao$^+$}
\email{yaoquanming@4paradigm.com}
\affiliation{
	\institution{4Paradigm Inc., Beijing, China}
}
\author{Wei-Wei Tu$^*$}
\email{tuweiwei@4paradigm.com}
\affiliation{
	\institution{4Paradigm Inc., Beijing, China}
}
\author{Yuqiang Chen}
\email{chenyuqiang@4paradigm.com}
\affiliation{
	\institution{4Paradigm Inc., Beijing, China}
}
\author{Qiang Yang}
\email{qyang@cse.ust.hk}
\affiliation{
	\institution{Hong Kong University of Science and Technology, Hong Kong}
}
\author{Wenyuan Dai}
\email{daiwenyuan@4paradigm.com}
\affiliation{
	\institution{4Paradigm Inc., Beijing, China}
}










\begin{abstract}

Feature crossing captures interactions among categorical features and is useful to enhance learning from tabular data in real-world businesses. In this paper, we present AutoCross, an automatic feature crossing tool provided by 4Paradigm to its customers, ranging from banks, hospitals, to Internet corporations. By performing beam search in a tree-structured space, AutoCross enables efficient generation of high-order cross features, which is not yet visited by existing works. Additionally, we propose successive mini-batch gradient descent and multi-granularity discretization to further improve efficiency and effectiveness, while ensuring simplicity so that no machine learning expertise or tedious hyper-parameter tuning is required. Furthermore, the algorithms are designed to reduce the computational, transmitting, and storage costs involved in distributed computing. Experimental results on both benchmark and real-world business datasets demonstrate the effectiveness and efficiency of AutoCross. It is shown that AutoCross can significantly enhance the performance of both linear and deep models.
\footnote{ Y. Luo and M. Wang make equal contribution to this work; 
	Q. Yao and W.-W. Tu are the corresponding authors.}

\end{abstract}

%
%



\maketitle

\bibliographystyle{ACM-Reference-Format}

\section{Introduction}
\label{sec:intro}

Recent years have seen the emergence of the automated machine learning (AutoML) \cite{yao2018taking,zhang2019}
as a promising way to make machine learning techniques easier to use, so that  
the manpower heavily involved in the current applications could be spared,
and
greater social and commercial benefits could be made.
In particular, AutoML aims to automate some parts or the whole of the machine learning pipeline, which usually comprises data preprocessing, feature engineering, model selection, hyper-parameter tuning, and model training.
It has been well recognized that the performance of machine learning methods depends greatly on the quality of features~\cite{tom1997machine,liu1998feature,domingos2012few}.
Since raw features rarely lead to satisfying results,  
manual feature generation is often carried out to better represent the data and improve the learning performance.
However, it is often a tedious and task-specific work.
This motivates automatic feature generation~\cite{rosales2012post,cheng2014gradient,chapelle2015simple,juan2016field,katz2016explorekit,blondel2016higher,qu2016product,zhang2016deep,lian2018xdeepfm}, one major topic of AutoML, where informative and discriminative features are automatically generated.
In 4Paradigm, a company with the aspiration to make machine learning techniques accessible to more people,
we also make efforts on this topic. 
We provide simple yet powerful feature generation tools to organizations and companies without machine learning expertise, and enable them to better exploit their data.

The customers of 4Paradigm range from banks, hospitals, to various Internet Corporations.
In their real-world businesses, e.g., 
fraud detection~\cite{bolton2002statistical,wang2010comprehensive},
medical treatment~\cite{kononenko2001machine}, 
online advertising~\cite{zeff1999advertising,evans2009online} or recommendation~\cite{bobadilla2013recommender},
etc.,
the majority of gathered data is presented in the form of tables, 
where each row corresponds to an instance and each column to a distinct feature.
Such data is often termed as \textit{tabular data}.
Furthermore, in such data, a considerable part of features are \textit{categorical}, e.g., `\code{job}' and
`\code{education}'
to describe the occupation and education status of an individual, respectively.
These features, indicating an entity or describing some important attributes, are very important and informative.
In order to make better use of them,
many feature generation methods have been proposed recently~\cite{rosales2012post,chapelle2015simple,cheng2016wide,juan2016field,qu2016product,zhang2016deep,guo2017deepfm,lian2018xdeepfm}.
In this paper, we also aim to increase learning performance by exploiting categorical features.


\textit{Feature crossing}, taking cross-product of sparse features, is a promising way to capture the interaction among categorical features and is widely used to enhance learning from tabular data in real-world businesses
~\cite{rosales2012post,chapelle2015simple,cheng2016wide,wang2017deep,lian2018xdeepfm}.
The results of feature crossing are called \textit{cross features}~\cite{wang2017deep,lian2018xdeepfm}, or conjunction features~\cite{rosales2012post,chapelle2015simple}.
Feature crossing represents the co-occurrence of features, which may be highly correlated with the target label.
For example, the cross feature 
`\code{job $\otimes$ company}' indicates that an individual takes a specific job in a specific company, and is a strong feature to predict one's income.
Feature crossing also adds nonlinearity to data, which may improve the performance of learning methods.
For instance, the expressive power of linear models is restricted by their linearity, but can be extended by cross features~\cite{rosales2012post,chapelle2015simple}.
Last but not least, explicitly generated cross features are highly interpretable, which is an appealing characteristic in many real-world businesses, such as medical treatment and fraud detection.

\begin{figure}[t]
	\centering
	\includegraphics[width=1\columnwidth]{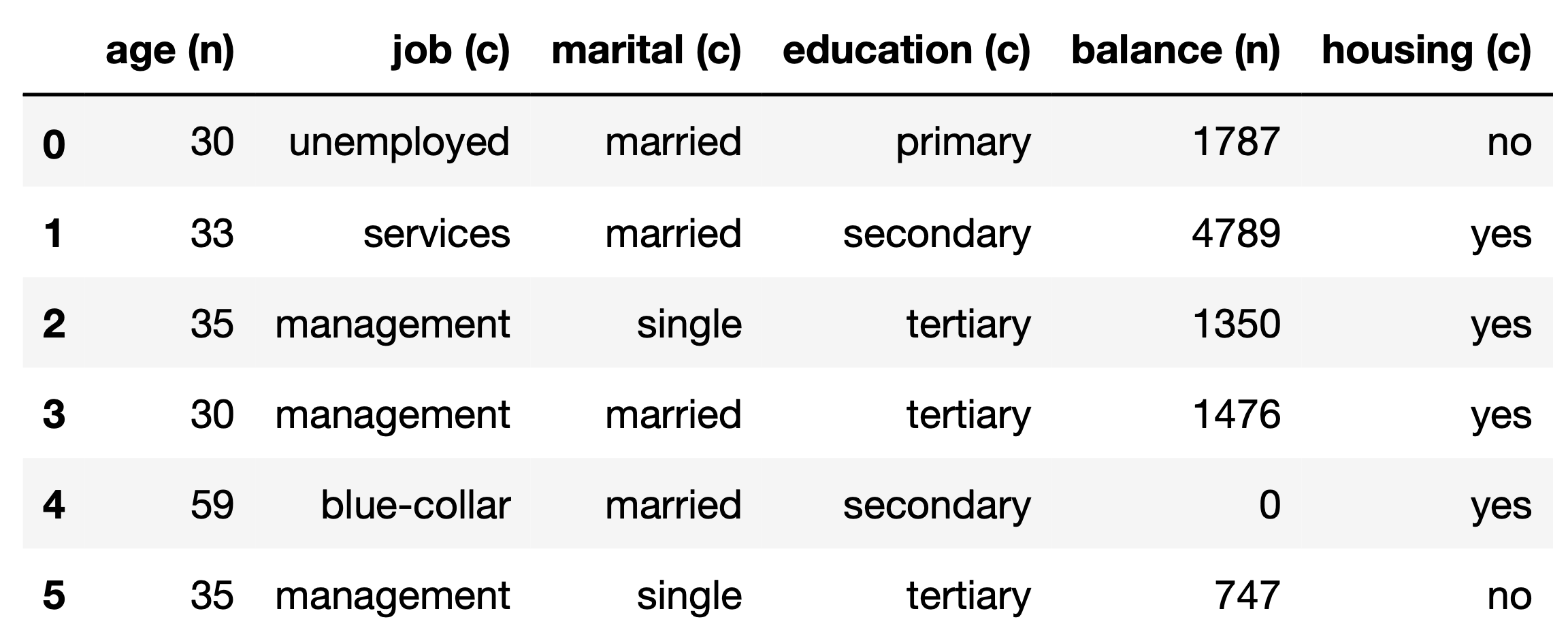}
	\caption{An example of tabular data (UCI-Bank). The letters embraced in the parentheses indicate the feature types, with `n' standing for `numerical'  and `c' for `categorical'.}
	\label{fig:tabular}
	\vspace{-15px}
\end{figure}

However, to enumerate all cross features may lead to degraded learning performance, since they may be irrelevant or redundant, introduce noise, and increase the difficulty of learning.
Hence, only useful and important cross features should be generated, but they are often task-specific~\cite{rosales2012post,chapelle2015simple,wang2017deep,lian2018xdeepfm}.
In traditional machine learning applications, human experts are heavily involved in feature engineering, striving to generate useful cross features for every task, with their domain-knowledge in a trial-and-error manner.
Furthermore, even experienced experts may have trouble when the number of original features is large.
The manpower requirement and difficulty of manual feature crossing greatly increase the total cost to apply machine learning techniques, and even hinder many companies and organizations to make good use of their data.

This raises the great demand for \textit{automatic feature crossing}, the target of our work presented in this paper.
In addition to our main target, i.e., to automate feature crossing with high effectiveness and efficiency, we need to consider several more requirements:
1) \textit{simplicity requirement}: a tool with high simplicity is user-friendly and easy to use.
The performance of most existing automatic feature generation methods greatly depends on some hyper-parameters.
Examples are the thresholds in ExploreKit~\cite{katz2016explorekit}, subsampling ratio in~\cite{chapelle2015simple}, and network architectures in deep-learning-based methods~\cite{ cheng2016wide,qu2016product,zhang2016deep,guo2017deepfm,lian2018xdeepfm}.
These hyper-parameters should be properly determined or carefully tuned by the user for each specific task. 
Since no machine learning expertise of our customers is assumed,
hyper-parameters that require careful setting or fine-tuning should be avoided.
2) \textit{distributed computing}: the large amount of data and features in real-world businesses makes distributed computing a must. Feature crossing methods should take into consideration the corresponding computational, transmitting and storage costs.
3) \textit{real-time inference requirement}: real-time inference is involved in many real-world businesses. 
In such cases, once an instance is inputted, the feature should be produced instantly and the prediction made.
Latency requirement for real-time inference is rigorous~\cite{guo2002survey,crankshaw2017clipper}, which raises the need for fast feature producing.

To summarize our business requirements, we need our automatic feature crossing tool to have high \textit{effectiveness}, \textit{efficiency} and \textit{simplicity}, be optimized for \textit{distributed computing}, and enable \textit{fast inference}.
To address these challenges, we present AutoCross, an automatic feature crossing tool especially designed for tabular data with considerable categorical features.
The major contributions of this paper are summarized as follows:

\begin{itemize}
	\item
	We propose an efficient AutoML algorithm to explicitly search for useful cross features in an extensive search space. It enables to construct \textit{high-order} cross features, which can further improve the learning performance but is not yet visited by existing works.
	
	
	\item 
	AutoCross features high simplicity with minimized exposure of hyper-parameters.
	We propose successive mini-batch gradient descent and multi-granularity discretization. They improve the efficiency and effectiveness of feature crossing while avoiding careful hyper-parameter setting.
	
	\item 
	AutoCross is fully optimized for distributed computing.
	By design, 
	the algorithms can reduce the computational, transmitting, and storage costs.

	\item Extensive experimental results on both benchmark and real-world business datasets are reported to verify the effectiveness and efficiency of AutoCross. It can significantly improve the performance of generalized linear models, while keeping a low inference latency.
	It is also shown that AutoCross can accompany deep models, by which means we can combine the advantage of explicit and implicit feature generation and further improve the learning performance.

\end{itemize}
With the these characteristics, AutoCross enables our customers to make better use of their data in real-world businesses with little learning and using costs.


The remaining of this paper is organized as follows: in Section~\ref{sec:motivation}, we demonstrate what motivates us to develop our own feature crossing tool; next, in Section~\ref{sec:overview}, the overall system is introduced; the techniques employed in each component, as well as the choices made in designing the system, are detailed in Section~\ref{sec:method}; experimental results are presented in Section~\ref{sec:exp}; Section~\ref{sec:related} reviews some related works and Section~\ref{sec:con} concludes this paper.

\section{Motivation}
\label{sec:motivation}

\begin{table*}[t]
	\centering
	\caption{Comparison between AutoCross and other feature generation methods for tabular data.}
	\label{tab:compare}
	\vspace{-5px}
	\scalebox{0.8}{
		\begin{tabular}{ c | C{80px} C{80px} C{80px} C{80px}}
			\hline\hline
			\textbf{Method}                                                                  & \textbf{High-order Feature Cross} & \textbf{Simplicity} & \textbf{Fast Inference} & \textbf{Interpretability}   \\ \hline
			Search-based methods (e.g., \cite{rosales2012post,chapelle2015simple})           & $\times$                          & medium              & $\surd$                 & $\surd$                     \\ \hline
			Implicit deep-learning-based methods (e.g., \cite{zhang2016deep,qu2016product})  & $\times$                           & low                 & $\times$                & $\times$                    \\ \hline
			Explicit deep-learning-based methods (e.g., \cite{wang2017deep,lian2018xdeepfm}) & $\times$                           & low                 & $\times$                & $\surd$                    \\ \hline\hline
			\textbf{AutoCross}                                                               & $\surd$                           & high                & $\surd$                 & $\surd$                    \\ \hline\hline
		\end{tabular}
	}
\end{table*}

Cross features are constructed by vectorizing the cross-product ($\otimes$) of original features: 
\begin{equation}
\label{eq:conj}
\vect{c}_{i,j,\cdots,k} = \code{vec}\left(\vect{f}_i \otimes \vect{f}_j \otimes \cdots\otimes \vect{f}_k\right),
\end{equation}
where $\vect{f}_i$'s are binary feature vectors (generated by one-hot encoding or hashing trick) and $\code{vec}(\cdot)$ vectorizes a tensor. A cross feature is also a binary feature vector. If a cross feature uses three or more original features, we denote it as a \textit{high-order} cross feature.
In this section, we state motivation of our work: why we consider high-order cross features and why existing works do not suit our purpose.

While most early works of automatic feature generation focus on second-order interactions of original features~\cite{rosales2012post,cheng2014gradient,chapelle2015simple,juan2016field,katz2016explorekit},
trends have appeared to consider higher-order (i.e., with order higher than two) interactions to make data more informative and discriminative~\cite{blondel2016higher,qu2016product,zhang2016deep,lian2018xdeepfm}.
High-order cross features, just like other high-order interactions, can further improve the quality of data and increase predictive power of learning algorithms.
For example, a third-order cross feature `\code{item $\otimes$ time $\otimes$ region}' can be a strong feature
to recommend regionally preferred food during certain festivals.
However, explicit generation of high-order cross features is not yet visited in existing works.
In the remaining of this section, we demonstrate that existing feature generation approaches either do not generate high-order cross features or cannot fulfill our business requirements.



On the one hand, \textit{search-based feature generation methods} employ explicit search strategies to construct useful features or feature sets. 
Many such methods focus on numerical features~\cite{smith2005genetic,fan2010generalized,kanter2015deep,katz2016explorekit,tran2016genetic}, and do not generate cross features.
As for existing feature crossing methods~\cite{rosales2012post,chapelle2015simple}, they are not designed, and are hence inefficient, to perform high-order feature crossing.


On the other hand, there are \textit{deep-learning-based feature generation approaches}, where interactions among features are implicitly or explicitly represented by specific networks.
Variants of deep learning models are proposed to deal with categorical features (e.g., Factorisation-machine supported neural networks~\cite{zhang2016deep} and Product-based neural networks ~\cite{qu2016product}).
Efforts are also made to accompany deep learning models with additional structures that incorporate: 1) manually designed features (e.g., Wide \& Deep~\cite{cheng2016wide}); 2) factorization machines (e.g., DeepFM~\cite{guo2017deepfm} and xDeepFM~\cite{lian2018xdeepfm}), and/or 3) other feature construction components (e.g., Deep \& Cross Network~\cite{wang2017deep} and xDeepFM~\cite{lian2018xdeepfm}).
Especially, xDeepFM~\cite{lian2018xdeepfm}, proven superior to other deep-learning-based approaches mentioned above, proposed a compressed interaction network~(CIN) to explicitly capture high-order interactions among embedded features. This is done by performing entry-wise product on them:
\begin{equation}
\label{eq:cin}
\vect{e}_{i,j,\cdots,k} = \vect{W}_i \vect{f}_i \circ \vect{W}_j \vect{f}_j \circ \cdots \circ \vect{W}_k \vect{f}_k,
\end{equation}
where $\circ$ denotes the entry-wise product (Hadamard product) and $\vect{W}_i$'s the embedding matrices so that $\vect{Wf}\in \mathbb{R}^D$. Different embedding matrices lead to different interaction $\vect{e}$'s. 
However, as stated in the following proposition, the resulting features in Equation~\eqref{eq:cin} is only a special case of embedded high-order cross features.
\begin{prop}
	\label{prop:cin}
	There exist infinitely many embedding matrices $\vect{C}$'s with $D$ rows so that: there do not exist any embedding matrices $\vect{A}$ and $\vect{B}$ that satisfy the following equation:
	\begin{equation}
		\vect{Ax}\circ\vect{By} = \vect{Cz},
	\end{equation}
	for all binary vectors $\vect{x}$, $\vect{y}$ and their crossing $\vect{z}$.
\end{prop}
The proof can be found in Appendix~\ref{app:cin}.
Furthermore, deep models are relatively slow in inference 
, which makes model compression~\cite{han2015deep} or other accelerating techniques necessary to deploy them in many real-time inference systems.

Motivated by the usefulness of high-order cross features and the limitation of existing works, in this paper, we aim to propose a new automatic feature crossing method that is efficient enough to generate high-order cross features, while satisfying our business requirements, i.e., to be simple, optimized for distributed computing, and enable fast inference.
Table~\ref{tab:compare} compares AutoCross and other existing methods.


\section{System Overview}
\label{sec:overview}

Figure~\ref{fig:system} gives an overview of AutoCross, comprising three parts: 1) the general work flow; 2) the component algorithms; and 3) the underlying infrastructure.

\begin{figure}[ht]
	\centering
	\includegraphics[width=1.0\columnwidth]{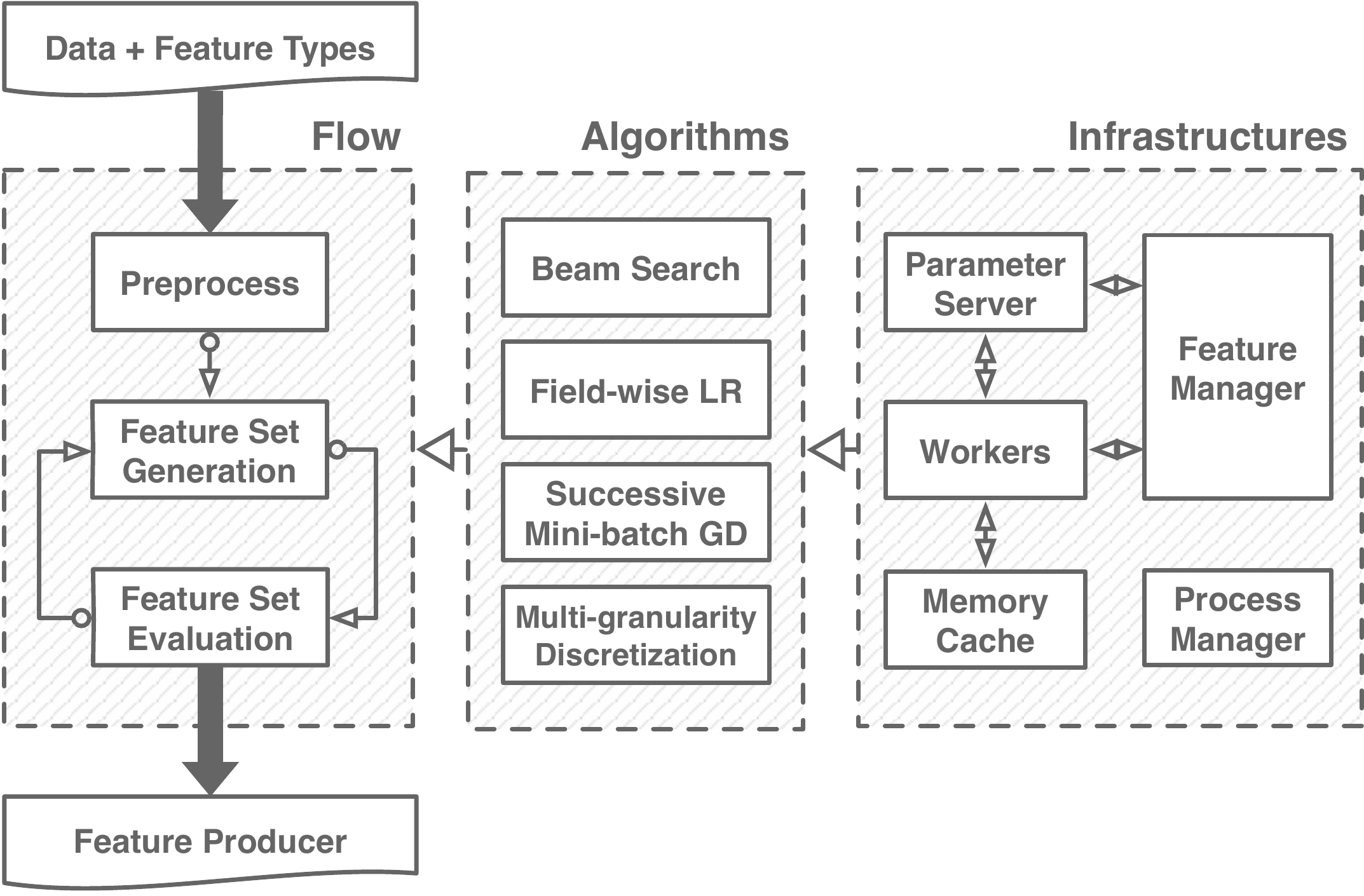}
	\vspace{-10px}
	\caption{System overview of AutoCross.}
	\label{fig:system}
\end{figure}

From the users' perspective, AutoCross is a black box that takes as input the training data and feature types (i.e., categorical, numerical, time series, etc.), and outputs a feature producer.
The feature producer can fast perform crossing learned by AutoCross to transform the input data,
which is then used by the learning algorithm in model training, or the learned model in inference.
It employs hashing trick~\cite{weinberger2009feature} to improve the accelerate feature producing.
Compared with deep-learning-based methods, the feature producer takes significantly less computation resources, and is hence especially suitable for real-time inference.

Inside the black box (`flow' part in Figure~\ref{fig:system}), the data will first be preprocessed, where
hyper-parameters are determined, missing values filled and numerical features discretized.
Afterwards, useful feature sets are iteratively constructed in a loop consisting of two steps: 1) \textit{feature set generation}, where candidate feature sets with new cross features are generated; and 2) \textit{feature set evaluation}, where candidate feature sets are evaluated and the best is selected as a new solution.
This iterative procedure is terminated once some conditions are met.

From the implementation perspective (`infrastructures' part in Figure~\ref{fig:system}), the foundation of AutoCross is a distributed computing environment based on the well-known parameter server~(PS) architecture~\cite{li2013parameter}.
Data is cached in memory by blocks, where each block contains a small subset of the training data.
Workers visit the cached data blocks, generate corresponding features, and evaluate them.
A feature manager takes control over the feature set generation and evaluation. 
A process manager controls the whole procedure of feature crossing, including hyper-parameter adaptation, data preprocessing, work flow control, and program termination.

The algorithms, that bridge the work flow and infrastructures, are the main focus of this paper (`algorithms' part of Figure~\ref{fig:system}). 
Each algorithm corresponds to a part in the work flow: 
we employ beam search for feature set generation to explore an extensive search space (Section~\ref{sec:gen}), field-wise logistic regression and successive mini-batch gradient descent for feature set evaluation (Section~\ref{sec:eval}), and multi-granularity discretization for data preprocessing (Section~\ref{sec:pre}).
These algorithms are chosen, designed, and optimized with the considerations of simplicity and costs of distributed computing, as will be detailed in the next section.

\section{Method}
\label{sec:method}

In this section, we detail the algorithms used in AutoCross.
	We focus on the binary classification problem. It is not only the subject of most existing works~\cite{rosales2012post,chapelle2015simple,cheng2016wide,katz2016explorekit,lian2018xdeepfm}, but also the most widely considered problem in real-world businesses~\cite{zeff1999advertising,kononenko2001machine,bolton2002statistical,evans2009online,wang2010comprehensive,bobadilla2013recommender}.

\subsection{Problem Definition}

For the ease of discussion, first we assume that all the original features are categorical.
The data is represented in the \textit{multi-field} categorical form~\cite{zhang2016deep,wang2017deep,lian2018xdeepfm}, where each \textit{field} is a binary vector generated from a categorical feature by encoding (one-hot encoding or hashing trick).
Given training data $\mathcal{D}_{TR}$, we split it into a sub-training set $\mathcal{D}_{tr}$ and a validation set $\mathcal{D}_{vld}$.
Then, we represent $\mathcal{D}_{tr}$ with a feature set $\mathcal{S}$, and with learning algorithm $\mathcal{L}$ learn a model $\mathcal{L}(\mathcal{D}_{tr}, \mathcal{S})$.
To evaluate this model, we represent the validation set $\mathcal{D}_{vld}$ with the same feature set $\mathcal{S}$ and calculate a metric 
$\mathcal{E}\left(\mathcal{L}(\mathcal{D}_{tr}, \mathcal{S}), \mathcal{D}_{vld}, \mathcal{S}\right)$,
which should be maximized.

Now, we formally define the \textit{feature crossing problem} as:
\begin{equation}
\label{eq:def}
	\max_{\mathcal{S}\subseteq A(\mathcal{F)}} \mathcal{E}\left(\mathcal{L}(\mathcal{D}_{tr}, \mathcal{S}), \mathcal{D}_{vld}, \mathcal{S}\right),
\end{equation}
where $\mathcal{F}$ is the original feature set of $\mathcal{D}_{TR}$, and $A(\mathcal{F)}$ is the set of all original features and possible cross features generated from $\mathcal{F}$.

\begin{figure}[t]
	\centering
	\includegraphics[width=0.7\columnwidth]{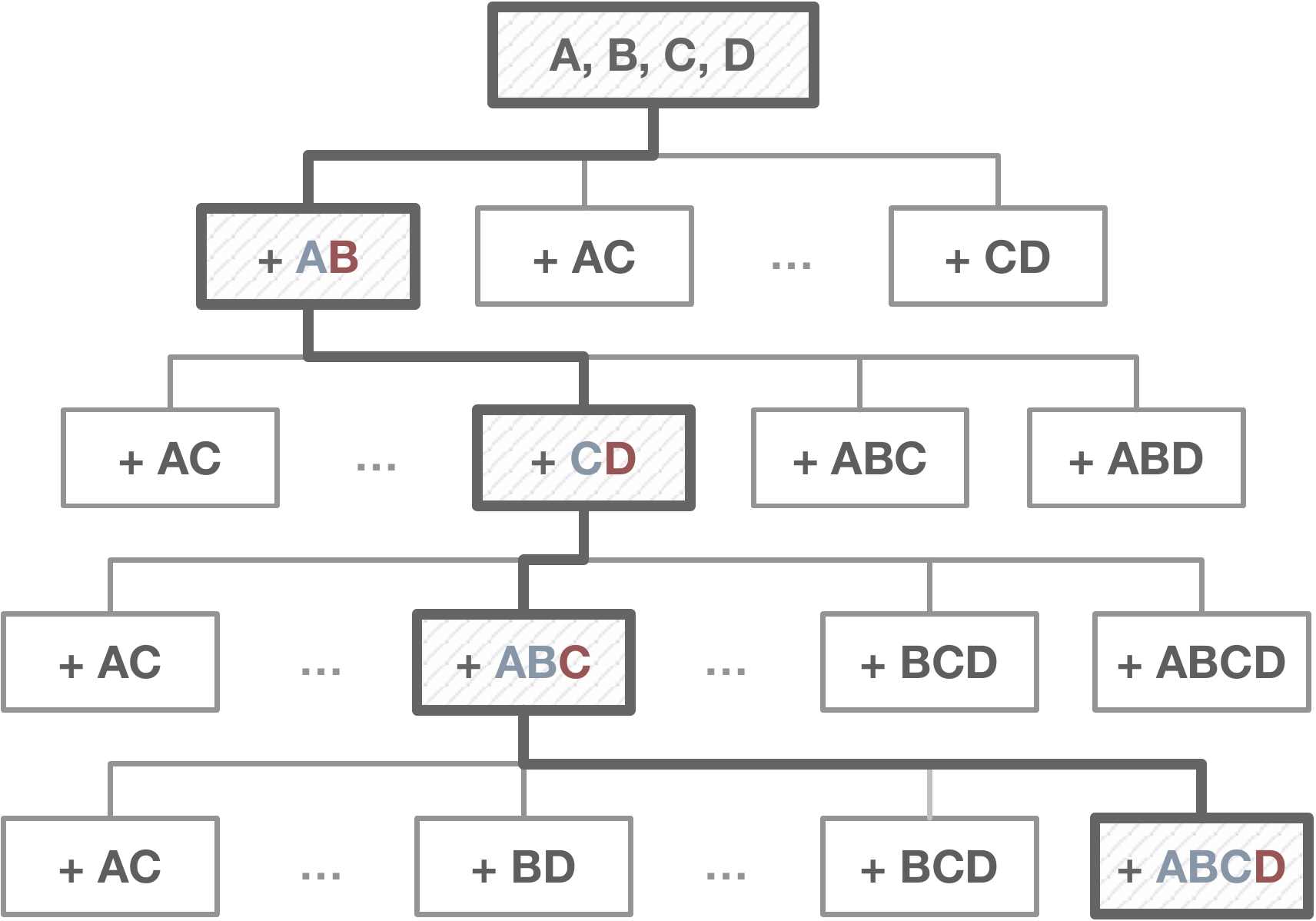}
	\caption{An illustration of the search space and beam search strategy employed in AutoCross. 
		In beam search, only the best node (bold stroke) at each level is expanded. We use two colors to indicate the two features that are used to construct the new cross feature.}
	\label{fig:beam}
\end{figure}

\subsection{Feature Set Generation}
\label{sec:gen}

In this subsection, we introduce the feature set generation method in AutoCross, which also determines the main search strategy.

We consider the feature crossing problem (Problem~\eqref{eq:def}).
Assume the size of the original feature set is $d$, which is also the highest order of cross features.
The size of $A(\mathcal{F})$ is:
\begin{equation}
	\mathrm{card}\left(A(\mathcal{F})\right) =  \sum_{k=1}^d C(d,k) =  2^d -1,
\end{equation}
and the number of all possible feature sets is $2^{(2^d-1)}$, a double exponential function of $d$.
Obviously, it is impractical to search for an globally optimal feature set in such an extensive space.
In order to find a moderate solution with limited time and computational resources, 
we employ a greedy approach to \textit{iteratively construct a locally optimal feature set}.

In AutoCross, we consider a tree-structured space $\mathcal{T}$ depicted in Figure~\ref{fig:beam}, where each node corresponds to a feature set and the root is the original feature set $\mathcal{F}$.
\footnote{
	In Figure~\ref{fig:beam} only one node at each level is expanded. This is because we use beam search strategy. It should be noted that the search space $\mathcal{T}$ is a fully expanded tree.}
For simplicity, in this example, we denote the crossing of two features \code{A} and \code{B} as \code{AB}, and higher-order cross features in similar ways.
For a node (a feature set), its each child is constructed by adding to itself one \textit{pair-wise} crossing of its own elements.
The pair-wise interactions between cross features (or a cross feature and an original feature) will lead to \textit{high-order} feature crossing. 
The new space $\mathcal{T}$ considers all possible features in $A(\mathcal{F})$, but excludes part of its  subsets.
With $\mathcal{T}$, to search for a feature set is equivalent to identifying a \textit{path} from the root of $\mathcal{T}$ to a specific node.
This can be done by iteratively adding cross features into a maintained feature set.
However, the size of $\mathcal{T}$ is $O\left((d^2/2)^k\right)$ where $k$ is the maximum number of generated cross features. 
It grows exponentially with $k$.
Hence, it will be extremely expensive to exhaustively visit all possible solutions, or in other words, to traverse $\mathcal{T}$.
To address this issue, we employ beam search to further improve the efficiency.

\begin{algorithm}[t]
	\small
	\caption{Feature Set Search Strategy in AutoCross.}
	\label{alg:beam}
	\begin{algorithmic}[1]
		\REQUIRE original feature set $\mathcal{F}$.
		\ENSURE solution $\mathcal{S}^*$.
		\STATE initialize current node $\mathcal{S}^* \leftarrow \mathcal{F}$;
		\WHILE {procedure not terminated}
		\STATE \textbf{Feature Set Generation}: expand $\mathcal{S}^*$, generate its children node set $\code{children}(\mathcal{S}^*)$ by adding to itself different pair-wise crossing of its elements;
		\STATE \textbf{Feature Set Evaluation}: evaluate all candidate feature sets in $\code{children}(\mathcal{S}^*)$ and identify the best child $\mathcal{S}'$;
		\STATE visit $\mathcal{S}'$: $\mathcal{S}^* \leftarrow \mathcal{S}'$
		\ENDWHILE
		
		\RETURN $\mathcal{S}^*$.
	\end{algorithmic}
\end{algorithm}

\textit{Beam search}~\cite{medress1977speech} is a greedy strategy to explore a graph with low computation and memory costs.
The basic idea is to only expand the most promising nodes during search.
First we generate all children nodes of the root, evaluate their corresponding feature sets and choose the best performing one to visit next.
In the process that follows, we expand the current node and visit its most promising child.
When the procedure is terminated, we end at a node that is considered as the solution.
By this means, we only consider $O\left(kd^2\right)$ nodes in a search space with size $O\left((d^2/2)^k\right)$, and the cost grows linearly with $k$, the maximal number of cross features.
It enables us to efficiently construct high-order cross features.
This feature set generation method leads to the main feature set search strategy in AutoCross, as described in Algorithm~\ref{alg:beam}.
Figure~\ref{fig:beam} highlights a search path that begins from the original feature set \code{\{A, B, C, D\}} and ends at \code{\{A, B, C, D, AB, CD, ABC, ABCD\}}, the solution. 

\subsection{Feature Set Evaluation}
\label{sec:eval}

A vital step in Algorithm~\ref{alg:beam} is to evaluate the performance of candidate feature sets (Step 4).
Here, the performance of a candidate set $\mathcal{S}$ is expressed as $\mathcal{E}\left(\mathcal{L}(\mathcal{D}_{tr}, \mathcal{S}), \mathcal{D}_{vld}, \mathcal{S}\right)$ (see Problem~\eqref{eq:def}), denoted as $\mathcal{E}(\mathcal{S})$ for short.
To directly estimate it, we need to learn a model with algorithm $\mathcal{L}$ on the training set represented by $\mathcal{S}$ and evaluate its performance on the validation set.
Though highly accurate, direct evaluation for feature sets is often rather expensive.
In real-world business scenarios,
training a model to convergence may take great computational resources.
Such direct evaluations are often too expensive to be invoked repetitively in the feature generation procedure.
In order to improve the evaluation efficiency, we proposed field-wise logistic regression and successive mini-batch gradient descent in AutoCross.

\subsubsection{Field-wise Logistic Regression}

Our first effort to accelerate feature set evaluation is \textit{field-wise logistic regression} (field-wise LR).
Two approximations are made.
First,
we use logistic regression~(LR) trained with mini-batch gradient descent to evaluate candidate feature sets, and use the corresponding performance to approximate the performance of the  learning algorithm $\mathcal{L}$ that actually follows.
We choose logistic regression since, as a generalized linear model, it is the most widely used model in large scale machine learning. It is simple, scalable, fast for inference, and makes interpretable predictions~\cite{rosales2012post,chapelle2015simple,cheng2016wide}.

The second approximation is that, during model training, we only learn the weights of the newly added cross feature, while other weights are fixed. 
Hence, the training is `field-wise'.
For example, assume the current solution feature set is $\mathcal{S}^* = \code{\{A, B, C, D\}}$, and we want to evaluate a candidate set $\mathcal{S} = \code{\{A, B, C, D, AB\}}$. 
Only the weights of \code{AB} is updated in training.
Formally, denote an instance as $\vect{x} = [\vect{x}_s^\mathrm{T}, \vect{x}_c^\mathrm{T}]^\mathrm{T}$, where $\vect{x}_s$ corresponds to all features in the current solution and $\vect{x}_c$ the newly added cross feature. Their corresponding weights are $\vect{w} = [\vect{w}_s^\mathrm{T}, \vect{w}_c^\mathrm{T}]^\mathrm{T}$.
An LR model makes prediction:
\begin{equation}
\label{eq:lr}
P(y=1|\vect{x})  =  s(\vect{w}^\mathrm{T}\vect{x}) = s(\vect{w}_s^\mathrm{T}\vect{x}_s + \vect{w}_c^\mathrm{T}\vect{x}_c) =s(\vect{w}_c^\mathrm{T}\vect{x}_c + b_{sum}),
\end{equation}
where $s(\cdot)$ is the sigmoid function.
In field-wise LR, we only update $\vect{w}_c$, and 
since we fix $\vect{w}_s$, $b_{sum}$ is a constant scalar during training.
We cache the values of $b_{sum}$ in the memory so that they can be directly fetched by the workers.

Figure~\ref{fig:field} shows how field-wise LR works on the parameter server architecture.
Field-wise LR improves the efficiency of feature set evaluation from several aspects:
1) \textit{Storage}: the workers store only $\vect{x}_c$ (in sparse format where only the hashed values are stored) and $b_{sum}$, rather than full representation of instances; there is a negligible overhead to store $b_{sum}$ in memory cache;
2) \textit{Transmitting}: the contents of transmission between the memory cache and workers are $b_{sum}$ and the hashed values of the features that are used to construct $\vect{x}_c$. Transmission of full instance representation is therefore spared;
3) \textit{Computation}: only $\vect{w}_c$ is updated, which reduces the computation burden of workers and parameter servers;
all workers directly fetch the stored $b_{sum}$, so that the latter need not to be repetitively calculated for every mini-batch at each worker.

\begin{figure}[t]
	\centering
	\includegraphics[width=0.8\columnwidth]{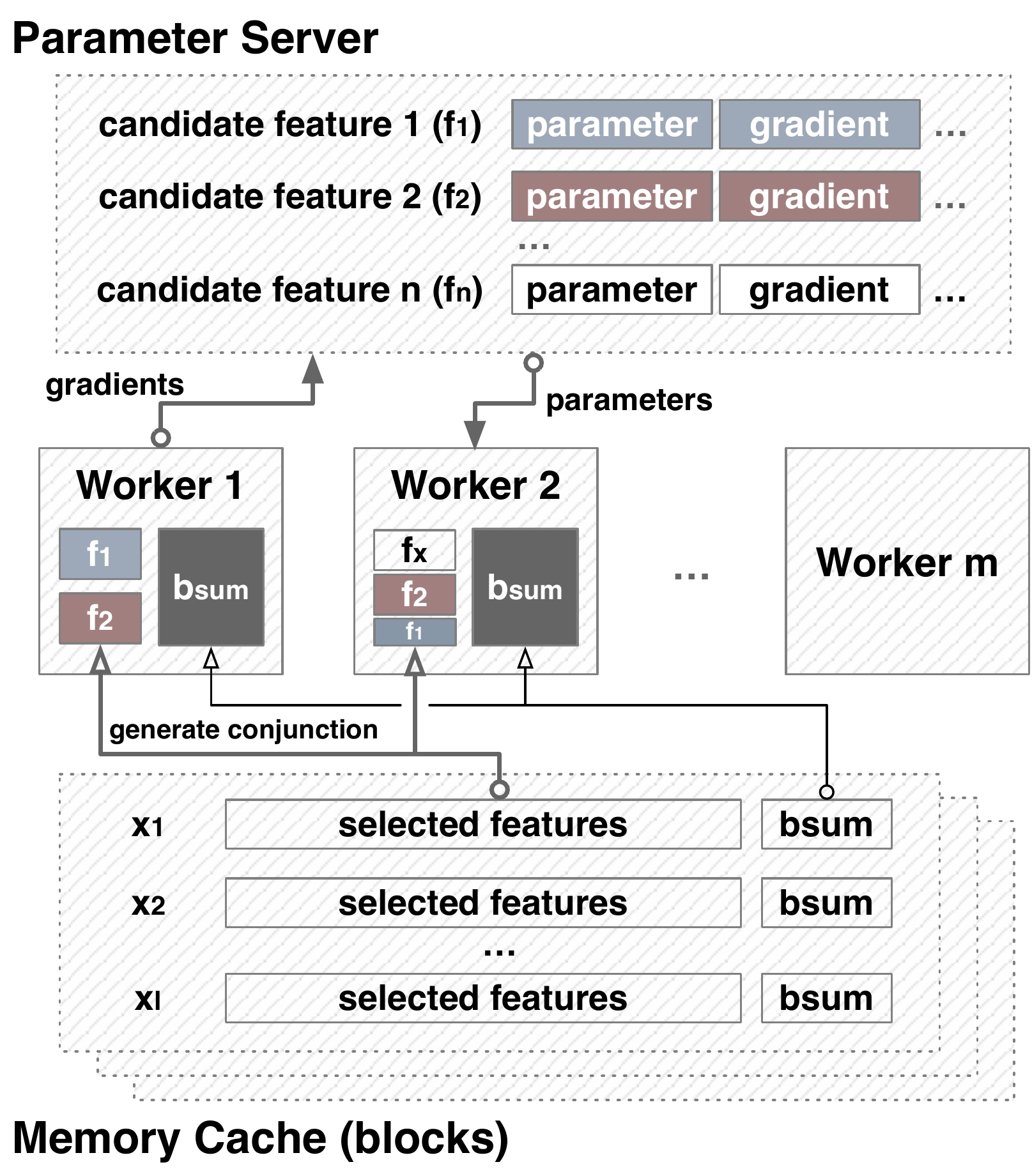}
	\vspace{-10px}
	\caption{Illustration of field-wise logistic regression for feature evaluation based on a parameter server architecture.
	}	
	\label{fig:field}
\end{figure}

Once the field-wise LR finishes, we estimate the performance of the resulting model on the validation set $\mathcal{D}_{vld}$.
We use the resulting metrics $\mathcal{E}'(\mathcal{S})$, such as Area-Under-Curve~(AUC), accuracy, or negative log-loss, to evaluate the quality of $\mathcal{S}$.
Obviously, 
$\mathcal{E}'(\mathcal{S})$ is an approximation of $\mathcal{E}(\mathcal{S})$, with accuracy traded for higher efficiency.
However, since the purpose of feature set evaluation is to \textit{identify the most promising candidate}, rather than \textit{to accurately estimate the performance of candidates}, 
a degraded accuracy is acceptable if only it can recognize the \textit{best} candidate with high probability.
Experimental results reported in Section~\ref{sec:exp} demonstrate the effectiveness of field-wise LR. 



After a candidate is selected to replace the current solution $\mathcal{S}^*$ (Step 6, Algorithm~\ref{alg:beam}), we train an LR model with the new $\mathcal{S}^*$, evaluate its performance, and update $b_{sum}$ for data blocks that will be used in the next iteration. Details will be discussed immediately.

\subsubsection{Successive Mini-batch Gradient Descent}

In AutoCross, we use a successive mini-batch gradient descent method to further accelerate field-wise LR training.
It is motivated by the successive halving algorithm~\cite{jamieson2016non} which was originally proposed for multi-arm bandit problems.
Successive halving features an efficient allocation of computing resources and high simplicity.
In our case, we consider each candidate feature set as an arm, and a pull of the arm is to train the corresponding field-wise LR model with a data block.
The instant reward of pulling an arm is the resulting validation AUC of the partially trained model.
The training data is equally split into $N \ge \sum_{k=0}^{\lceil\log_2n\rceil-1}2^k$ data blocks, where $n$ is the number of candidates.
Then we invoke Algorithm~\ref{alg:subsample} to identify the best candidate feature set.
Successive mini-batch gradient descent allocates more resources to more promising candidates.
The only hyper-parameter $N$, namely the number of data blocks, is adaptively chosen according to the size of data set and the working environment.
Users do not need to tune the mini-batch size and sample ratios that are critical for vanilla subsampling.

\begin{algorithm}[ht]
	\small
	\caption{Successive Mini-batch Gradient Descent (SMBGD).}
	\label{alg:subsample}
	\begin{algorithmic}[1]
		\REQUIRE set of candidate feature sets $\mathbb{S} = \{\mathcal{S}_i\}_{i=1}^n$, training data equally divided into $N \ge \sum_{k=0}^{\lceil\log_2n\rceil-1}2^k$ data blocks.
		\ENSURE best candidate $\mathcal{S}'$.
		\FOR{$k = 0, 1, \cdots, \lceil\log_2n\rceil-1$}
		\STATE use additional $2^k$ data blocks to update the field-wise LR models of all $\mathcal{S} \in \mathbb{S}$, with warm-starting;
		\STATE evaluate the models of all $\mathcal{S}$'s with validation AUC;
		\STATE keep the top half of candidates in $\mathbb{S}$: $\mathbb{S} \leftarrow \code{top\_half}(\mathbb{S})$ (rounding down);
		\STATE break if $\mathbb{S}$ contains only one element;
		\ENDFOR
		\RETURN $\mathcal{S}'$ (the singleton element of $\mathbb{S}$).
	\end{algorithmic}
\end{algorithm}

\subsection{Preprocessing}
\label{sec:pre}

In AutoCross, we use \textit{discretization} in the data preprocessing step to enable feature crossing between numerical and categorical features.
Discretization has been proven useful to improve predicting capability of numerical features~\cite{liu2002discretization,kotsiantis2006discretization,chapelle2015simple}.
The most simple and widely-used discretization method is equal-width discretization, i.e., to split the value range of a feature into several equal-width intervals.
However, in traditional machine learning applications, the number of intervals, named as \textit{granularity} in our work, has a great impact on the learning performance and should be carefully determined by human experts.

In order to automate discretization and  spare its dependence on human experts, we propose a \textit{multi-granularity discretization} method.
The basic idea is simple: instead of using a fine-tuned granularity, we discretize each numerical feature into several, rather than only one, categorical features, each with a different granularity.
Figure~\ref{fig:lfc} gives an illustration of discretizing a numerical feature with four levels of granularity.
Since more levels of granularity are considered, it is more likely to get a promising result.


In order to avoid the dramatic increase in feature number caused by discretization,
once these features are generated, we use field-wise LR (without considering $b_{sum}$) to evaluate them and keep only the best half.
	A remaining problem is how to determine the levels of granularity.
	For an experienced user, she can set a group of potentially good values.
	If no values are specified, AutoCross will use $\{10^{p}\}_{p=1}^P$ as default values, where
	$P$ is an integer determined by a rule-based mechanism that considers the available memory, data size and feature numbers.

\begin{figure}[t]
	\centering
	\includegraphics[width=0.86\columnwidth]{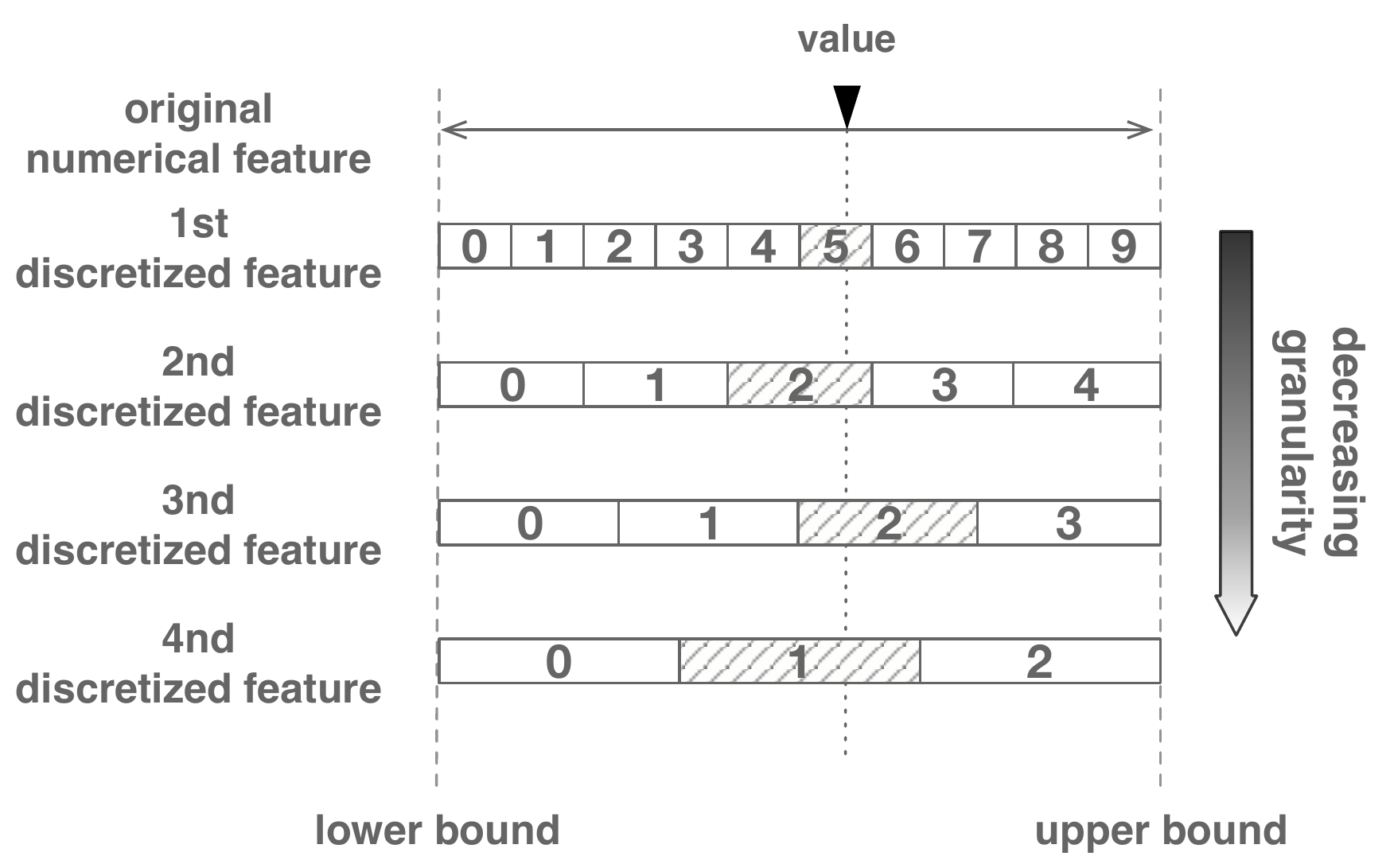}
	\vspace{-10px}
	\caption{An illustration of multi-granularity discretization. Shade  indicates the value taken by each discretized feature.}
	\vspace{-15px}
	\label{fig:lfc}
\end{figure}

In addition, AutoCross will invoke a tuning procedure in the preprocessing step to find optimal hyper-parameters for LR models. They will be used in all LR models involved subsequently.

\subsection{Termination}

Three kinds of termination conditions are used in AutoCross:
1) \textit{runtime condition}: the user can set a maximal runtime of AutoCross. When the time elapses, AutoCross terminates outputs the current solution $\mathcal{S}^*$. Additionally, the user can always interrupt the procedure and get the result of the time;
2) \textit{performance condition}: after a new feature set is generated (Step 6, Algorithm~\ref{alg:beam}), an LR model is trained with all its features. If, compared with the former set, the validation performance degrades, the procedure is terminated;
3) \textit{maximal feature number}: the user can give a maximal cross feature number so that AutoCross stops when the number is reached.

\section{Experiments}
\label{sec:exp}
In this section, we demonstrate the effectiveness and efficiency of AutoCross.
First, by comparing AutoCross with several reference methods on both benchmark and real-world business datasets, we show that with feature crossing it can significantly improve the performance of both linear and deep models,
and that high-order cross features are useful.
Then we report the time costs of feature crossing with AutoCross. Finally, we show the advantage of AutoCross in real-time inference.


\subsection{Setup}

\paragraph{Datasets:}
we test AutoCross with both benchmark and real-world business datasets, gathered from different  applications. 
Table~\ref{tab:benchmark}
summarizes these datasets\footnote{Availability of data sets are in Appendix~\ref{app:data}.}.
All the datasets are for binary classification. 
The real-world business datasets are provided by the customers of 4Paradigm with sanitization. 

\begin{table}[ht]
	\centering
	\caption{Characteristics of datasets used in the experiments.
		`Num.' and `Cate.' indicate numerical and categorical features respectively. `\# Val.' indicates the number of different values taken by the categorical features. `H.R.' is short for `human resource' and `Adv.' for `advertising'.}
	\label{tab:benchmark}
	\vspace{-10px}
	\scalebox{0.8}
	{\begin{tabular}{c | c | c | c | c | c | c}
			\hline\hline
			\multicolumn{7}{c}{Benchmark Datasets}\\
			\hline
			\multirow{2}{*}{Name} & \multicolumn{2}{c|}{\# Samples} & \multicolumn{3}{c|}{\# Features} &\multirow{2}{*}{Domain} \\ \cline{2-6}
			& Training   & Testing               & \# Num. & \# Cate.        & \# Val.    &           \\ \hline
			Bank      & 27,459     & 13,729                & 10   & 10    &    63      &      Banking     \\
			Adult      & 32,561     & 16,281                & 6    & 8      &     42    &     Social      \\
			Credit     & 100,000    & 50,000                & 10   & 0        &   0    &     Banking      \\
			Employee & 29,493      &  3,278                 & 0  & 9  &  7,518     & H. R. \\
			Criteo     & 41,256 K & 4,584 K             & 13   & 26        &  33,762 K    &     Adv.      \\ \hline\hline
			\multicolumn{7}{c}{Real-World Business Datasets}\\\hline
			\multirow{2}{*}{Name} & \multicolumn{2}{c|}{\# Samples} & \multicolumn{3}{c|}{\# Features} & \multirow{2}{*}{Domain} \\ \cline{2-6}
			& Training   & Testing               & \# Num. & \# Cate.        & \# Val.   &          \\ \hline
			Data1      & 2,641,185    & 719,998              & 34   & 28    &     4,181,854    & Sports \\
			Data2      & 1,888,366    & 1,119,778            & 8    & 19     &     109,180   &  Talkshow \\
			Data3     & 2,340,209    & 1,059,016             & 55   & 21        &     3,174,081 &  Social \\
			Data4     & 2,848,746      &  688,481             & 7  & 19  &  455,778    & Video \\
			Data5     & 11,802,126      &  2,058,424         & 8  & 18  &  436,361    & News \\
			\hline\hline
	\end{tabular}}
\vspace{-10px}
\end{table}

\paragraph{Methods:}
in order to demonstrate the effectiveness of AutoCross, we compare the following methods;
\begin{itemize}
	\item \textbf{AC+LR}: logistic regression with cross features generated by AutoCross;
	\item \textbf{AC+W\&D}: Wide \& Deep method~\cite{cheng2016wide} whose wide part uses cross features generated by AutoCross;
	\item \textbf{LR (base)}: our self-developed logistic regression with only the original features. It is used as the baseline;
	\item \textbf{CMI+LR}: logistic regression with cross features generated by the method proposed in~\cite{chapelle2015simple}, where conditional mutual information~(CMI) is used as the metric to evaluate features. This method only considers second-order feature crossing;
	\item \textbf{Deep}: a deep model with embedding layers to deal with categorical features. It implicitly generate feature interactions;
	\item \textbf{xDeepFM}: the method proposed in~\cite{lian2018xdeepfm}, which explicitly generates features with a compressed interaction network. It is the state-of-the-art of deep-learning-based method.
\end{itemize}
In these methods, \textbf{AC+LR} and \textbf{AC+W\&D} use the cross features generated by AutoCross, and demonstrate its effectiveness to enhance linear and deep models.
\textbf{CMI+LR} uses a representative search-based feature crossing method.
\textbf{xDeepFM} is the state-of-the-art method following the Wide \& Deep framework. We choose it as a reference method since it outperforms other existing deep-learning-based methods, as reported in~\cite{lian2018xdeepfm}.
We also consider \textbf{Deep} to test how a bare-bone deep model performs.
All these methods are designed to handle tabular data with categorical features.

\paragraph{Reproducibility:}
the features and models are learned with training and validation data (20\% of the training data, if needed), and the resulting AUCs on the testing data indicate the performance of different methods.
\textit{
More information about the settings of
methods under test can be found in Appendix~\ref{app:setting}.
}


\subsection{Results}

\subsubsection{Effectiveness}
Table~\ref{tab:res} reports the resulting test AUCs on the benchmark and real-world business datasets.
We did not run \textbf{CMI+LR} on real-world business datasets because there are multi-value categorical features that cannot be handled by CMI.
In the table, we highlighted the top-two methods for each dataset.
As can be easily observed, \textbf{AC+LR} ranks top-two in most cases, and often outperforms deep-learning-based methods (\textbf{Deep} and \textbf{xDeepFM}). \textbf{AC+W\&D} also shows competitive performance, demonstrating the capability of AutoCross to enhance deep models.
In most cases, \textbf{AC+LR} and \textbf{AC+W\&D} show better results than \textbf{xDeepFM}. According to Proposition~\ref{prop:cin}, xDeepFM only generates a special case of embedded cross feature.
This results show the effectiveness to directly and explicitly generate high-order cross features.

\begin{table}[ht]
	\centering
	\caption{Experimental results (test AUC) on benchmark and real-world business datasets.
}
	\label{tab:res}
	\vspace{-10px}
	\scalebox{0.85}
	{\begin{tabular}{c| c  c  c  c  c }
			\hline\hline
			\multicolumn{6}{c}{Benchmark Datasets}\\
			\hline
			Method & Bank & Adult & Credit & Employee & Criteo \\\hline
			LR (base) & 0.9400 & 0.9169 & 0.8292 & 0.8655 & 0.7855 \\
			\textbf{AC+LR} & \textbf{0.9455} & \textbf{0.9280} & \textbf{0.8567} & \textbf{0.8942} & 0.8034\\
			\textbf{AC+W\&D} & 0.9420 & \textbf{0.9260} & \textbf{0.8623} & \textbf{0.9033} & \textbf{0.8068} \\
			CMI+LR & \textbf{0.9431} & 0.9153 & 0.8336 & 0.8901 & 0.7844\\
			Deep & 0.9418 & 0.9130 & 0.8369 & 0.8745 & 0.7985 \\
			xDeepFM & 0.9419 & 0.9131 & 0.8358 & 0.8746 & \textbf{0.8059} \\
			\hline\hline
			\multicolumn{6}{c}{Real-World Business Datasets}\\
			\hline
			Method & Data1 & Data2 & Data3 & Data4 & 
			Data5 
			\\\hline
			LR (base) & 0.8368 & 0.8356 & 0.6960 & 0.6117 & 
			0.5992  
			\\
			\textbf{AC+LR} & \textbf{0.8545} & \textbf{0.8536} & \textbf{0.7065} & \textbf{0.6276} & 
			0.6393
			\\
			\textbf{AC+W\&D} & \textbf{0.8531} & \textbf{0.8552} & \textbf{0.7026} & \textbf{0.6260} &
			\textbf{0.6547}
			\\
			Deep & 0.8479 & 0.8463 & 0.6936 & 0.6207 &
			0.6509  
			\\
			xDeepFM & 0.8504 & 0.8515 & 0.6936 & 0.6241 &
			\textbf{0.6514}  
			\\
			\hline\hline
	\end{tabular}}
\end{table}


As has been discussed in the papers of Wide \& Deep~\cite{cheng2016wide} and DeepFM~\cite{guo2017deepfm}, in online recommendation scenarios,
small improvement (0.275\% in \cite{cheng2016wide} and 0.868\% in \cite{guo2017deepfm}, compared with LR) in off-line AUC evaluation can lead to a significant increase in online CTR and hence great commercial benefits.
Table~\ref{tab:auc} shows the test AUC improvement brought by AutoCross.
It can be observed that both \textbf{AC+LR} and \textbf{AC+W\&D} achieve significant improvement over \textbf{LR (base)}, and \textbf{AC+W\&D} also considerably improve the performance of deep model.
These results demonstrate that by generating cross features, AutoCross can make the data more informative and discriminative, and improve the learning performance.
The promising results achieved by AutoCross also demonstrate the capability of field-wise LR to identify useful cross features.

\begin{table}[ht]
	\centering
	\caption{Test AUC improvement v.s. LR (base) and Deep.
	}
	\label{tab:auc}
	\vspace{-10px}
	\scalebox{0.85}
	{\begin{tabular}{c  c  c  c  c | c}
			\hline\hline
			\multicolumn{6}{c}{\textbf{AC+LR} v.s. \textbf{LR (base)}}\\
			\hline
			Bank & Adult & Credit & Employee & Criteo & \textbf{Average} \\
			0.585\% & 1.211\% & 3.316\% & 3.316\% & 2.279\%  & 2.141\%\\
			\hline
			Data1 & Data2 & Data3 & Data4 & 
			Data5 
			& \textbf{Average}\\
			2.115\% & 2.154\% & 1.509\% & 2.599\% & 
			6.692\% 
			&  
			3.014\%	
			\\
			\hline\hline
			\multicolumn{6}{c}{\textbf{AC+W\&D} v.s. \textbf{LR (base)}}\\\hline
			Bank & Adult & Credit & Employee & Criteo & \textbf{Average} \\
			0.213\% & 0.992\% & 3.992\% & 4.367\% & 2.712\% & 
			2.455\% \\\hline
			Data1 & Data2 & Data3 & Data4 & 
			Data5 
			& \textbf{Average}\\
			1.948\% & 2.346\% & 0.948\% & 2.338\%  &
			9.546\%		  
			& 
			3.368\% \\
			\hline\hline
			\multicolumn{6}{c}{\textbf{AC+W\&D} v.s. \textbf{Deep}}\\\hline
			Bank & Adult & Credit & Employee & Criteo & \textbf{Average} \\
			0.021\% & 1.424\% & 3.035\% & 3.293\% & 1.039\% & 1.763\% \\\hline
			Data1 & Data2 & Data3 & Data4 & 
			Data5 
			& \textbf{Average}\\
			0.6133\% & 1.0516\% & 1.2976\% & 0.8539\% & 0.5361\%
			& 0.880\%
			 \\
			\hline\hline
	\end{tabular}}
\vspace{-10px}
\end{table}


\subsubsection{The effect of high-order features.}
With the above reported results, we have demonstrated the effect of AutoCross.
Figure~\ref{fig:dist} shows the number of second/high-order cross features generated for each dataset, where the latter take a considerable proportion.
Besides, in Table~\ref{tab:secvshigh},
we compare the performance improvements brought by \textbf{CMI+LR}, that only generates second-order cross features, and \textbf{AC+LR} that considers high-order feature crossing.
We can see that \textbf{AC+LR} stably and constantly outperforms \textbf{CMI+LR}.
This result demonstrates the usefulness of high-order cross features.

\begin{figure}[ht]
	\centering
	\vspace{-10px}
	\includegraphics[width=0.8\columnwidth]{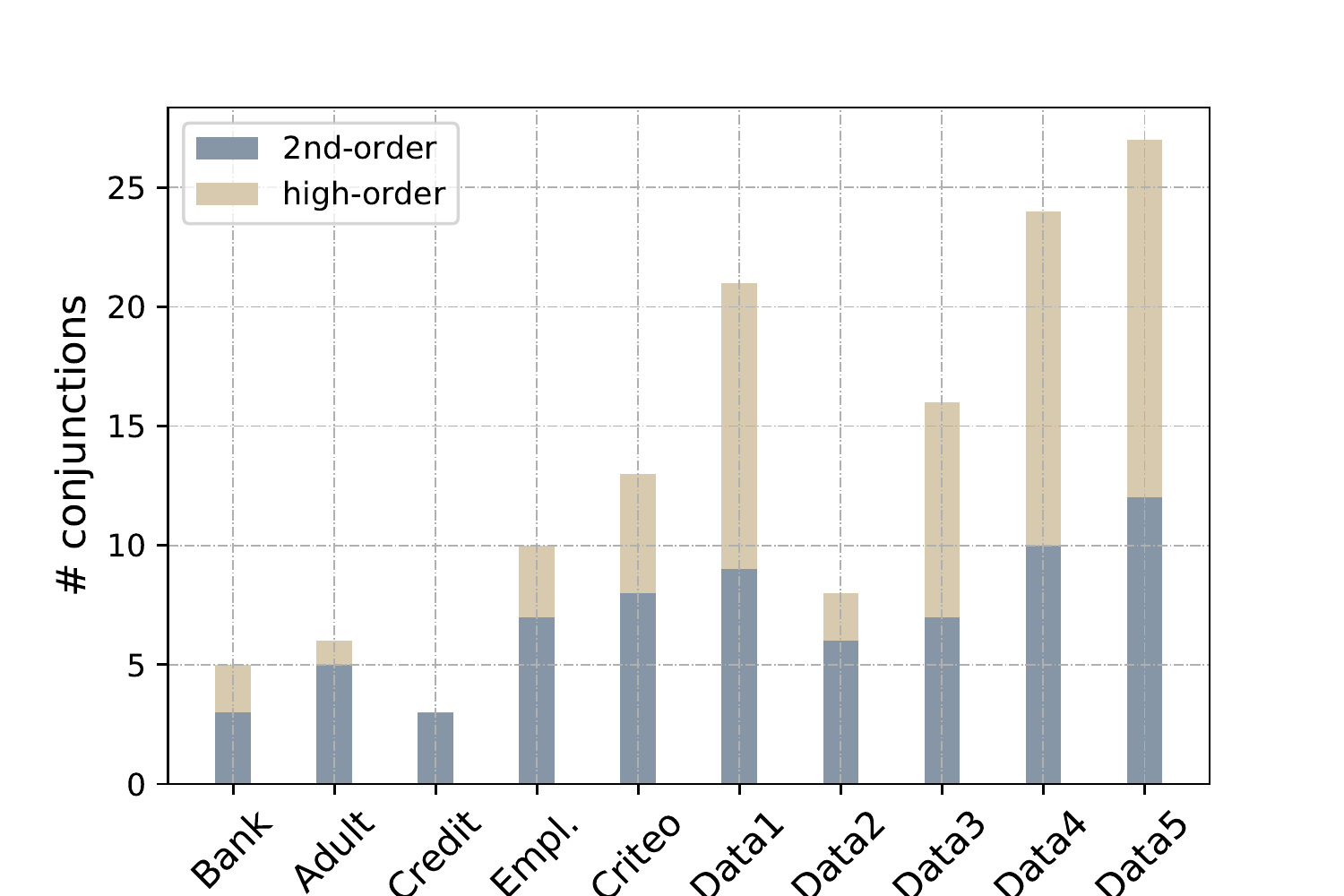} 
	\vspace{-5px}
	\caption{The number of second/high-order cross features generated for each dataset.}
	\label{fig:dist}
	\vspace{-10px}
\end{figure}

\begin{table}[ht]
\centering
\caption{Test AUC improvement: second v.s. high order features on benchmark datasets.}
\vspace{-10px}
\scalebox{0.80}
{\begin{tabular}{c| c  c  c  c  c | c}
	\hline\hline
	$\!\!$v.s. \textbf{LR(base)}$\!\!$ & Bank    & $\!\!$Adult$\!\!$    & Credit  & $\!\!$Employee$\!\!$ & $\!\!$Criteo   & $\!\!$\textbf{Average}$\!\!$ \\ \hline
	         \textbf{CMI+LR}           & 0.330\% & $\!\!$-0.175\% & 0.531\% & 2.842\%  & $\!\!$-0.140\% & 0.678\%                      \\ \hline
	          \textbf{AC+LR}           & 0.585\% & 1.211\%  & 3.316\% & 3.316\%  & 2.279\%  & 2.141\%                      \\
	          \hline\hline
\end{tabular}}
\label{tab:secvshigh}
\end{table}

\subsubsection{Time costs of feature crossing.}
Table~\ref{tab:time} reports the feature crossing time of AutoCross on each dataset.
Figure~\ref{fig:auccurve} shows the validation AUC (\textbf{AC+LR}) versus runtime on real-world business datasets.
Such curves are visible to the user and she can terminate AutoCross at any time to get the current result.
It is notable that due to the high simplicity of AutoCross, no hyper-parameter needs to be fine-tuned, and the user does not need to spend any extra time to get it work.
In contrast, if deep-learning-based methods are used, plenty of time will be spent on the network architecture design and hyper-parameter tuning.

\begin{table}[ht]
	\centering
	\caption{Cross feature generation time (unit: hour).}
	\label{tab:time}
	\vspace{-10px}
	\scalebox{0.8}
	{\begin{tabular}{c c c c c  }
			\hline\hline
			\multicolumn{5}{c}{Benchmark Datasets}\\
			\hline
			Bank & Adult & Credit & Employee & Criteo  \\\hline
			0.0267 & 0.0357 & 0.3144 & 0.0507 & 3.0817  \\
			\hline\hline
			\multicolumn{5}{c}{Real-World Business Datasets}\\
			\hline
			Data1 & Data2 & Data3 & Data4 & 
			Data5  
			\\\hline
			0.9327 & 0.7973 & 1.5206 & 2.7572 & 
			5.1861  
			\\
			\hline\hline
	\end{tabular}}
\end{table}

\begin{figure}[t]
	\centering
	\vspace{-5px}
	\includegraphics[width=0.8\columnwidth]{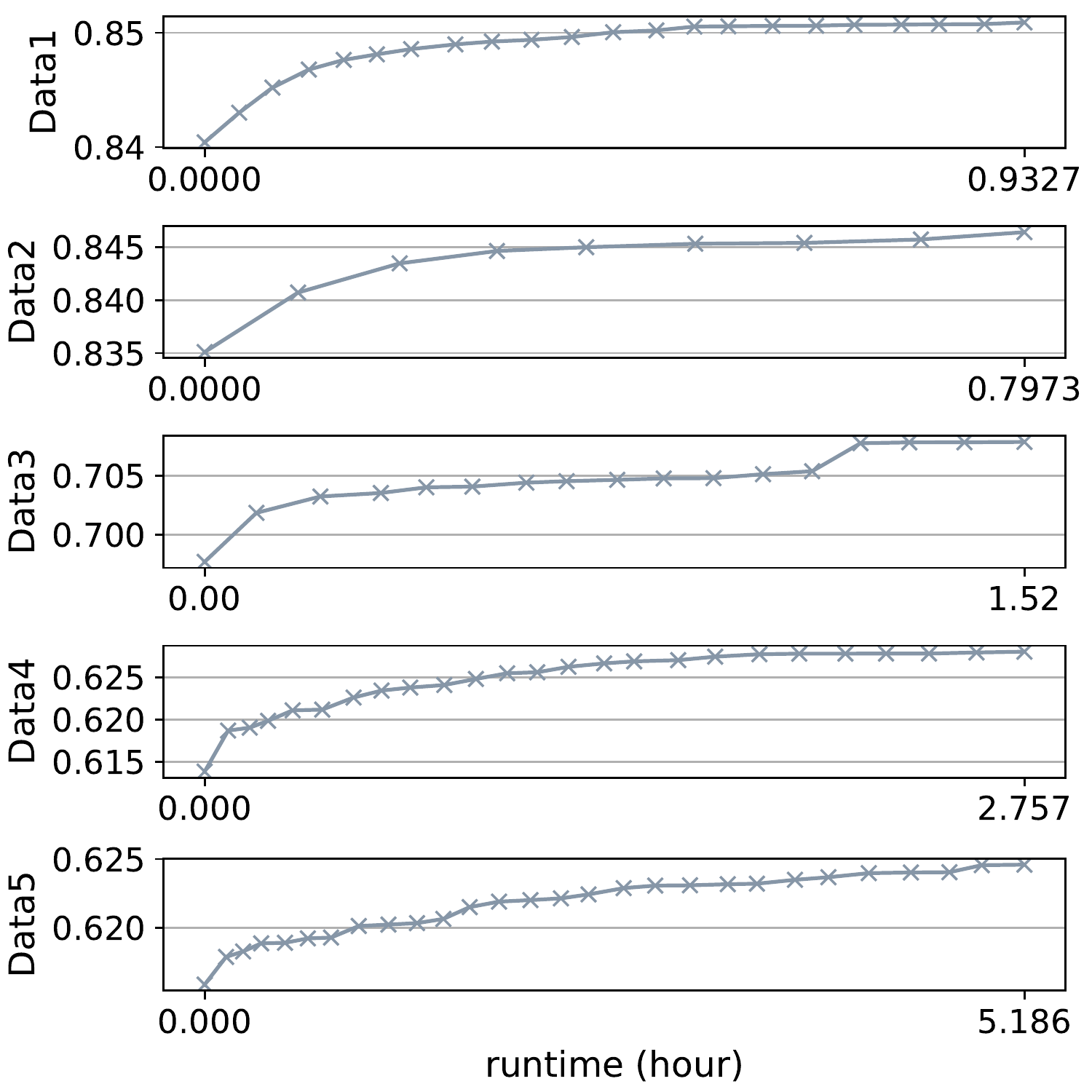} 
	\vspace{-5px}
	\caption{Validation AUC curves in real-business datasets.}
	\vspace{-10px}
	\label{fig:auccurve}
\end{figure}




\subsubsection{Inference Latency}


In many real-world businesses, the application scenario of a feature generation tool comprises three stages: 1) off-line feature generation;
2) off-line/online model training; 3) online inference.
In this scenario, the off-line generation stage is invoked the least frequently, for instance, features can be generated weekly or even monthly.
In contrast, within every millisecond, hundreds or thousands of inferences may sequentially take place, which makes high efficiency a must.
Online inference consists of two major steps: 1) feature producing to transform the input data, and 2) inference to make prediction.
Deep-learning method combines these steps.
In Table~\ref{tab:inference}, we report the inference time of \textbf{AC+LR}, \textbf{AC+W\&D}, \textbf{Deep} and \textbf{xDeepFM}.

\begin{table}[ht]
	\centering
	\caption{Inference latency comparison (unit: millisecond).}
	\label{tab:inference}
	\vspace{-10px}
	\scalebox{0.85}
	{\begin{tabular}{c | c c c c c  }
			\hline\hline
			\multicolumn{6}{c}{Benchmark Datasets}\\
			\hline
			Method & Bank & Adult & Credit & Employee & Criteo  \\\hline
			\textbf{AC+LR} & 0.00048 & 0.00048 & 0.00062 & 0.00073 & 0.00156  \\
			AC+W\&D & 0.01697 & 0.01493 & 0.00974 & 0.02807 & 0.02698  \\
			Deep & 0.01413 & 0.01142 & 0.00726 & 0.02166 & 0.01941 \\
			xDeepFM & 0.08828 & 0.05522 & 0.04466 & 0.06467 & 0.18985  \\
			\hline\hline
			\multicolumn{6}{c}{Real-World Business Datasets}\\
			\hline
			Method & Data1 & Data2 & Data3 & Data4 & 
			Data5  
			\\\hline
			\textbf{AC+LR} & 0.00367 & 0.00111 & 0.00185 & 0.00393 &
			0.00279  
			\\
			AC+W\&D & 0.03537 & 0.01706 & 0.04042 & 0.02434 & 
			0.02582  
			\\
			Deep & 0.02616 & 0.01348 & 0.03150 & 0.01414 & 
			0.01406  
			\\
			xDeepFM & 0.32435 & 0.11415 & 0.40746 & 0.12467& 
			0.13235  
			\\
			\hline\hline
	\end{tabular}}
\end{table}

It can be easily observed that \textbf{AC+LR} is orders of magnitude faster than other methods in inference.
This demonstrates that, AutoCross can not only improve the model performance, but also ensure fast inference with its feature producer.


%
%
%
%
%


\section{Related Works}
\label{sec:related}
In this section, we briefly review works that are loosely related to AutoCross and demonstrate why they do not suit our purpose.

Factorization machines seek low-dimensional embeddings of original features, and capture their interactions~\cite{cheng2014gradient,juan2016field,blondel2016higher}.
Such interactions, however, are not explicitly constructed. Furthermore, they may over-generalize~\cite{cheng2016wide} and/or introduce noise since they enumerate all possible interactions regardless of their usefulness~\cite{lian2018xdeepfm}.

There are also some embedded feature selection/generation methods, such as group lasso~\cite{meier2008group} and gradient boost machine~\cite{friedman2001greedy}, that intrinsically identify or implicitly construct useful features along model training.
However, these methods often struggle to deal with large scale problems with high-dimensional sparse data generated from categorical features, and/or have computational issues when the number of features is large, which happens when high-order feature crossing is considered.

Finally, itemsets~\cite{agrawal1993mining} 
have been well studied in data mining communities. Like cross features, they also represent the co-occurrence of attributes. However, the difference is that the elements in an itemset are often of a same kind, e.g., all being commodities. Also, itemsets are mostly used in rule-based machine learning techniques such as frequent patterns~\cite{han2000mining} and association rules~\cite{ng1998exploratory}. These techniques may have trouble to generalize, and are slow in inference when the number of rules is large, due to great retrieving costs~\cite{han2011data}.

\section{Conclusion}
\label{sec:con}
In this paper, we present AutoCross, an automatic feature crossing method for tabular data in real-world applications.
It captures useful interactions among categorical features and increases the predictive power of learning algorithms.
It employs beam search to efficiently construct cross features, which enables the consideration of high-order feature crossing, which is not yet visited by existing works.
Successive mini-batch gradient descent and multi-granularity discretization are proposed to further improve the efficiency and effectiveness while keeping high simplicity.
All the algorithms are designed for distributed computing to deal with big data in real-world businesses.
Experimental results show that AutoCross can significantly enhance learning from tabular data, outperforming other search-based and deep-learning-based feature generation methods dedicated to the same topic.

\bibliographystyle{ACM-Reference-Format}
{\small \bibliography{bib}}

\clearpage

\appendix

\section{Proof of Proposition~\ref{prop:cin}}
\label{app:cin}

%
%
Since high-order interactions can be represented by pair-wise interaction of lower-order ones, without loss of generality, in Proposition~\ref{prop:cin} we only consider second-order feature crossing $\vect{c}$ (Equation~\ref{eq:conj}) and second-order entry-wise product of embedded vectors $\vect{e}$ (Equation~\ref{eq:cin}).
\begin{proof}[Proof of Proposition~\ref{prop:cin}]
	First, we consider a weak version of Proposition~\ref{prop:cin}:
	\begin{prop}
		\label{prop:cin1}
		There exist at least one embedding matrix $\vect{C}$ with $D$ rows so that: there do not exist any embedding matrices $\vect{A}$ and $\vect{B}$ that satisfy the following equation:
		\begin{equation}
		\vect{Ax}\circ\vect{By} = \vect{Cz},
		\end{equation}
		for all binary vectors $\vect{x}$, $\vect{y}$ and their crossing $\vect{z}$.
	\end{prop}
	We proof Proposition~\ref{prop:cin1} by contradiction. 
	Consider its opposite proposition:
	\begin{prop}[Opposite of Proposition~\ref{prop:cin1}]
		\label{prop:neg}
		For all embedding matrices $\vect{C}$ with $D$ rows, there exist embedding matrices $\vect{A}$ and $\vect{B}$ that satisfy the following equation:
		\begin{equation}
		\vect{Ax}\circ\vect{By} = \vect{Cz},
		\end{equation}
		for all binary vectors $\vect{x}$, $\vect{y}$ and their crossing $\vect{z}$.	
	\end{prop}
	For simplicity, here we consider the cases where both $\vect{x}$ and $\vect{y}$ only have one hot bit, i.e., only one entry is `1' and others are `0'.
	Proposition~\ref{prop:neg} is correct in other cases unless it is true in this case.
	
	Let the $i$-th bit of $\vect{x}$ and $k$-th bit of $\vect{y}$ are `1'. To ease discussion, we denote the hot bit of the resulting crossing $\vect{c}$ as its $ik$-th bit.
	Further, denote $\vect{a}_i$ as the $i$-th column of $\vect{A}$, $\vect{b}_k$ the $k$-th column of $\vect{B}$, and $\vect{c}_{ik}$ the $ik$-th column of $\vect{C}$.
	Proposition~\ref{prop:neg} necessarily leads to that, 
	for any $\vect{C}$, we can find $\vect{A}$ and $\vect{B}$ to satisfy:
	\begin{equation}
		\vect{a}_i \circ \vect{b}_k = \vect{c}_{ik},
	\end{equation}
	for all $i$ and $k$.
	Now, consider two instances of $\vect{x}$, with the $i$- and $j$-th bits set as `1', respectively, and similarly two instances of $\vect{y}$, with the $k$- and $l$-th bits set as `1', respectively.
	We have four resulting equations:
	\begin{eqnarray*}
		\vect{a}_i \circ \vect{b}_k  =  \vect{c}_{ik}, &
	    \vect{a}_i \circ \vect{b}_l  =  \vect{c}_{il}, \\
	    \vect{a}_j \circ \vect{b}_k  =  \vect{c}_{jk}, &
	    \vect{a}_j \circ \vect{b}_l  =  \vect{c}_{jl}.
	\end{eqnarray*}
	
	Since $\vect{C}$ is an arbitrary embedding matrix, we can easily make it satisfy:
	\begin{equation}
		\label{eq:conf1}
		\vect{c}_{ik} / \vect{c}_{jk} \neq \vect{c}_{il} / \vect{c}_{jl},
	\end{equation}
	where $/$ denotes the entry-wise division of vectors.
	Equation~\eqref{eq:conf1} leads to:
	\begin{eqnarray*}
		& \vect{c}_{ik} / \vect{c}_{jk} =   (\vect{a}_i \circ \vect{b}_k) / (\vect{a}_j \circ \vect{b}_k) =
		\vect{a}_i / \vect{a}_j \\
		\neq &	\vect{c}_{il} / \vect{c}_{jl} =  (\vect{a}_i \circ \vect{b}_l) / (\vect{a}_j \circ \vect{b}_l)
		=  \vect{a}_i / \vect{a}_j.
	\end{eqnarray*}
	It leads to $\vect{a}_i / \vect{a}_j \neq \vect{a}_i / \vect{a}_j$ which apparently does not hold.
	Hence, Proposition~\ref{prop:neg} is false, which proofs that its opposite, namely Proposition~\ref{prop:cin1}, is true.
	
	Furthermore, we can construct infinitely many $\vect{C}$'s to satisfy Equation~\ref{eq:conf1}.
	Every such $\vect{C}$ falsifies Proposition~\ref{prop:neg}, and hence makes Proposition~\ref{prop:cin1} true.
	This verifies Proposition~\ref{prop:cin}.
\end{proof}

\section{Details of Experimental Setup}
\label{app:setting}

\subsection{Experimental Environment}
All experiments are carried out on a workstation with Intel(R) Xeon(R) CPU (E5-2630 v4 @ 2.20GHz, 24 cores), 
256G memory and 8T hard disk.

\subsection{Setup}

\paragraph{AutoCross Setup}
The hyper-parameters of AutoCross are the number of data blocks $N$ in successive mini-batch gradient descent, the levels of granularity for multi-granularity discretization, and the termination condition.
The first two are determined by a rule-based mechanism.
As a result, $N$ are set to $2\sum_{k=0}^{\lceil\log_2n\rceil-1}2^k$ for relatively small datasets, namely \code{Bank}, \code{Adult}, \code{Credit} and \code{Employee}.
This indicates that at most $50\%$ of the training data will be used in successive mini-batch gradient descent.
For other datasets, $N = 5\sum_{k=0}^{\lceil\log_2n\rceil-1}2^k$, which corresponds to a maximum of $20\%$ sample ratio.
With respect to the levels of granularity, AutoCross uses $\{10^i\}_{i=1}^3$ for all datasets.
As for the termination condition, we only invoke the performance condition, i.e., AutoCross terminates only if newly added cross feature leads to a performance degradation.

\paragraph{Logistic Regression Setup:}
logistic regression model is used in \textbf{LR (base)}, \textbf{AC+LR}, and \textbf{CMI+LR} methods.
The feature set evaluation of AutoCross also uses LR models.
We use our self-developed LR method in the experiments.
There are only three hyper-parameters: learning rate $\alpha\in [0.005, 1]$, L1 penalty $\lambda_1\in[1e-4,10]$, and L2 penalty $\lambda_2\in[1e-4,10]$.
In our experiment, as well as the real-world application of AutoCross,
we invoke a hyper-parameter tuning procedure before feature generation.
Log-grid search is used to find the optimal hyper-parameters.
They are used in \textbf{LR (base)}, \textbf{AC+LR}, and \textbf{CMI+LR} methods, as well as the LRs in AutoCross.

\paragraph{CMI+LR Setup:}
we only test \textbf{CMI+LR} on benchmark datasets since CMI cannot handle multi-value categorical features. For its feature generation method~\cite{chapelle2015simple},
we use the multi-granularity discretization to convert numerical features.
In order to ensure the accuracy of feature evaluation, we use all training data to estimate CMI.
An exception is the \code{Criteo} dataset, for which we set the subsample  ratio to $10\%$.
We set the maximal cross feature number to 15.

\paragraph{Deep Model Setup:}
we use the open-source implementation of xDeepFM (\textit{https://github.com/Leavingseason/xDeepFM}) in \textbf{xDeepFM} method, and use the deep component in \textbf{AC+W\&D} and \textbf{Deep} methods.
Hyper-parameters are set as the xDeepFM paper~\cite{lian2018xdeepfm} suggested.
To be more specific, we use $0.001$ as the learning rate,  Adam~\cite{kingma2014adam} with mini-batch size 4096 as the optimization method, 0.0001 as the L2 regularization penalty, 400 layers for deep component, 200/100 layers for compressed interaction network for \code{Criteo}/other datasets, and 10 as the dimension of field embedding.
Since validation data is not need in \textbf{xDeepFM} and \textbf{Deep}, we do not split the training set and use all of it in model training.

\subsection{Data Sets Availability}
\label{app:data}

\begin{itemize}[leftmargin = 10px]
\item adult: 

\url{https://archive.ics.uci.edu/ml/datasets/adult}

\item Bank: 

\url{https://www.kaggle.com/brijbhushannanda1979/bank-data}

\item Credit: 

\url{https://www.kaggle.com/c/GiveMeSomeCredit/data}

\item Employee: 

\url{https://www.kaggle.com/c/amazon-employee-access-challenge/data}

\item Criteo: 

\url{https://www.kaggle.com/c/criteo-display-ad-challenge/data}
\end{itemize}

Due to secrecy agreement,
real-world business datasets are not public available.

\end{document}